\documentclass[conference]{IEEEtran}

\usepackage{cite}
\usepackage{amsmath,amssymb,amsfonts}
\usepackage{graphicx}
\usepackage{textcomp}
\usepackage{xcolor}

\usepackage{graphicx}  
\usepackage{subcaption} 

\usepackage{multirow}
\usepackage{amsthm}
\usepackage{booktabs}
\usepackage{tabularx}
\usepackage{algorithm}
\usepackage{algorithmic}
\usepackage[switch]{lineno}
\newtheorem{lemma}{Lemma}
\usepackage{url}

\theoremstyle{definition}
\newtheorem{definition}{Definition}[section]

\def\BibTeX{{\rm B\kern-.05em{\sc i\kern-.025em b}\kern-.08em
    T\kern-.1667em\lower.7ex\hbox{E}\kern-.125emX}}
\begin{document}

\title{Boosting Time Series Prediction of Extreme Events by Reweighting and Fine-tuning}

\author{\IEEEauthorblockN{Jimeng Shi}
\IEEEauthorblockA{
\emph{Florida International University} \\
jshi008@fiu.edu}
\and
\IEEEauthorblockN{Azam Shirali}
\IEEEauthorblockA{
\emph{Florida International University} \\
ashir018@fiu.edu}
\and
\IEEEauthorblockN{Giri Narasimhan}
\IEEEauthorblockA{
\emph{Florida International University} \\
giri@fiu.edu}
}

\maketitle

\begin{abstract}
Extreme events are of great importance since they often represent impactive occurrences. For instance, in terms of climate and weather, extreme events might be major storms, floods, extreme heat or cold waves, and more.
However, they are often located at the tail of the data distribution. 
Consequently, accurately predicting these extreme events is challenging due to their rarity and irregularity.
Prior studies have also referred to this as the \emph{out-of-distribution} (OOD) problem, which occurs when the distribution of the test data is substantially different from that used for training.
In this work, we propose two strategies, \emph{reweighting} and \emph{fine-tuning}, to tackle the challenge. 
Reweighting is a strategy used to force machine learning models to focus on extreme events, which is achieved by a weighted loss function that assigns greater penalties to the prediction errors for the extreme samples relative to those on the remainder of the data.
Unlike previous intuitive reweighting methods based on simple heuristics of data distribution, we employ meta-learning to dynamically optimize these penalty weights.
To further boost the performance on extreme samples, we start from the reweighted models and fine-tune them using only rare extreme samples. 
Through extensive experiments on multiple data sets, we empirically validate that our meta-learning-based reweighting outperforms existing heuristic ones, and the fine-tuning strategy can further increase the model performance.
More importantly, these two strategies are model-agnostic, which can be implemented on any type of neural network for time series forecasting. 
The open-sourced code is available at \emph{\url{https://github.com/JimengShi/ReFine}}.
\end{abstract}


\section{Introduction}
\label{sec:intro}
Recently, deep learning (DL) has achieved unprecedented success in a variety of diverse applications \cite{krizhevsky2012imagenet}.
This success relies heavily on the availability of rich and high-quality datasets, i.e., large-scale datasets with a balanced distribution. 
In practice, most real-world datasets are imbalanced, necessitating a careful treatment of minority samples \cite{li2021autobalance}.
In time series, occurrences of extreme highs or lows are sparingly infrequent, leading to the emergence of long-tailed data distributions (e.g., extreme precipitation and extreme heat events).

\begin{figure}[ht]
\centering
    \begin{subfigure}[b]{0.23\textwidth}
        \includegraphics[width=\textwidth]{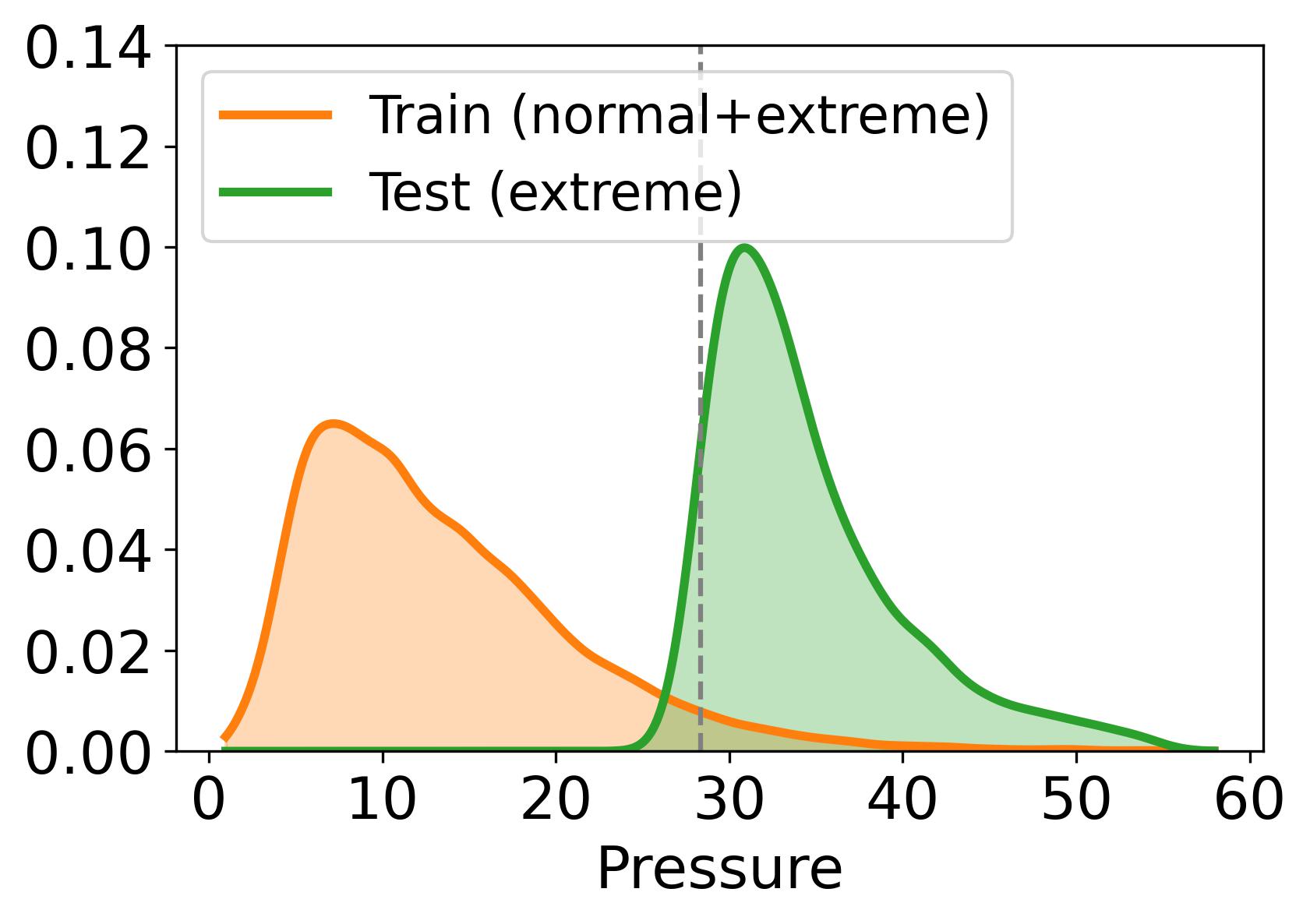}
    \label{fig:pdf_pressure}
    \end{subfigure}
    \begin{subfigure}[b]{0.24\textwidth}
        \includegraphics[width=\textwidth]{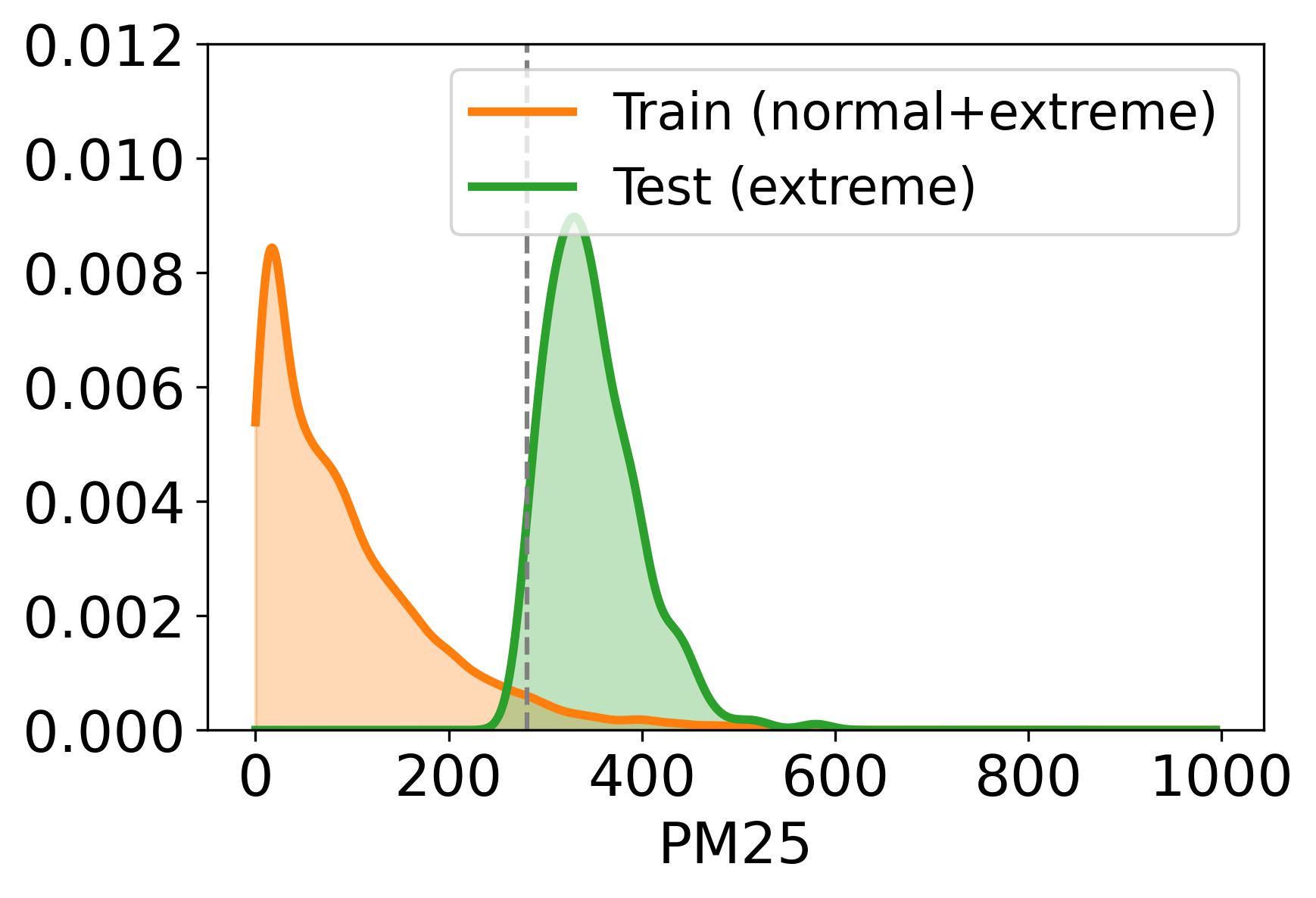}
    \label{fig:pdf_pm25}
    \end{subfigure}
\vspace{-2mm}
\caption{Out-of-Distribution (OOD) problem showing the disparity between the training and test sets. The gray dashed line represents the threshold (95$^{th}$ percentile) to separate normal and extreme samples.}
\label{fig:ood}
\end{figure}

Training with long-tailed datasets can bias against and also hide poor performance on the minority samples.
Most of the current DL models for time series prediction may perform poorly on extremes either during training - \textit{underfitting}, or during testing - \textit{overfitting} due to their rarity and irregularity.
Underfitting arises because DL models may lack sufficient exposure to minority knowledge during training, while overfitting occurs due to the out-of-distribution (OOD) problem, i.e., the disparity in the distributions between the training and test sets (see Figure \ref{fig:ood}). 
In some scenarios, training sets are often heavily skewed with an overwhelming majority of normal samples and a small minority of extreme samples, while the evaluation set may exclusively comprise rare extreme samples since they are of great interest in real datasets.
Theoretical analysis has shown that strong generalizations for OOD data cannot be truly achieved \cite{berend2020cats,lucas2020theoretical}, making it critical to seek effective methods to alleviate the problem.

A conventional approach to improving the performance on extreme events is \textbf{reweighting}. 
The key idea is to offset the imbalance in the data distribution by differentially weighting the prediction errors of normal and extreme samples in the loss function during training.
Such weighted loss functions allow predictive models to reduce the errors on extreme examples while preventing the abundance of normal samples from biasing the predictor.
Reweighting could be achieved by using \emph{heuristic} methods \cite{patterson2013parasite,ding2019modeling} based on prior knowledge of data imbalance or via \emph{meta-learning} \cite{chen2023meta,zhou2022model}.

\emph{Heuristic} reweighting may be achieved by using graded weights, with higher values for the errors in the minority and lower ones for the errors in the majority, or by weighting the groups inversely to the group size \cite{patterson2013parasite}.
Another approach is to model the Extreme Value distribution \cite{kozerawski2022taming,ding2019modeling}, assigning weights inversely to the probability of the data. 
The above methods assign weights based on some heuristics that require prior knowledge of data distributions, which are not always readily available.
In this paper, we implement and compare the existing two ways for reweighting. 
 
\emph{Meta-learning} offers an alternative approach for implementing reweighting, distinct from heuristic methods.
Instead of relying on prior knowledge of data distributions, meta-learning endeavors to learn penalty weights within the learning algorithm autonomously — learning to learn \cite{hospedales2021meta, huisman2021survey}.
A representative work determines the penalty weights by calculating the similarity between training and test samples \cite{chen2023meta}. While their learning strategy is dynamic, it is computationally expensive to compute the similarity for all test samples.
In our work, we achieve meta-learning-based reweighting with the help of a clean and unbiased evaluation set comprising solely extreme samples.
More specifically, we cast the reweighting task as a bilevel optimization problem \cite{franceschi2018bilevel}. In the inner loop, deep learning models are trained on weighted training samples. Meanwhile, in the outer loop, we minimize prediction errors on the preceding evaluation set to guide the learning of the best penalty weights.

\textbf{Fine-tuning} is a substantially different technique that takes existing models and boosts their performance on targeted tasks.
For example, many researchers expand the capabilities of pre-trained large models (LLMs) for specific applications and also achieve robust and stable performance \cite{houlsby2019parameter, zhang2023instruction}. 
Inspired by that, we hypothesize that our initially trained models, which perform well on massive normal events, can be subsequently fine-tuned to get a secondary model that focuses on performing well on rare extreme events. Such a two-phase solution is similar to the ones to address ``domain'' adaption tasks -- using models trained in one domain where there is enough annotated training data in another where there is little or none \cite{farahani2021brief, singhal2023domain}.

Our main contributions are summarized as follows:
\begin{itemize}
    \item To better model imbalanced data with rare extreme events during training, we apply two heuristic methods and adapt a meta-learning-based method to compute the penalty weights in the loss function to balance the bias created by normal data (majority) and to boost the learning from extreme (minority) samples.
    \item To further boost the model performance of time series prediction under extreme events, after reweighting, we subsequently incorporated a fine-tuning technique to adapt the previously trained model for extreme domain adaptation.
    \item We conduct extensive experiments across 4 datasets, which indicates the \emph{reweighting} and \emph{fine-tuning} methods can consistently outperform the previous benchmarks.
\end{itemize}
\section{Related work}
\label{sec:related_work}

\subsection{Time Series Prediction}
Traditional time series prediction employs linear methods like autoregressive moving averages \cite{said1984testing} or nonlinear approaches like NARX \cite{lin1996learning}. 
Nevertheless, the effectiveness of such methods is constrained due to their shallow architectures and low generalizability.
Over the past decades, deep learning has achieved significant success in various domains \cite{krizhevsky2012imagenet, stebliankin2023evaluating, chen2023tsmixer}. 
Representative work on time series prediction includes multilayer perceptron (MLP) \cite{chen2023tsmixer}, convolutional neural networks (CNNs) \cite{yang2015deep}, recurrent neural networks (RNNs) \cite{cho2014learning}, long-short term memory (LSTM) networks \cite{graves2012long}, graph neural networks (GNNs) \cite{wu2020connecting}, and well-designed transformer-based models \cite{zhou2021informer, zeng2023transformers, wu2021autoformer, nie2022time}. 
Despite their success, none of these models directly address the specific challenge of time series prediction for rare but vital extreme events, causing the distribution disparity between the training and test sets.

\subsection{Reweighting}
A potential solution to alleviate the poor performance of minority extreme samples is to differentially weight the prediction errors arising from the training samples.
For instance, simply assigning higher weights to the prediction errors of all minority samples and lower weights to those of all majority ones, intuitively up-weighting the rare group inversely to its group size \cite{patterson2013parasite}, or utilizing Extreme Value Theory (EVT) to up-weight the rare group inversely to the probability of long-tailed data \cite{kozerawski2022taming, zhang2021enhancing}. 
Zhang et al. \cite{zhang2021data} proposed a framework to integrate ML models with anomaly detection algorithms to filter extreme events and use percentile values as the weights. 
Li et al. \cite{li2023extreme} proposed, NEC+, learns extreme and normal predictions separately and assigns a corresponding probability as the weight for extreme and normal classes. Another work from them separates extreme and normal samples based on the distance to the mean value \cite{li2024learning}.
However, these existing methods compute the assigned weights using prior knowledge of the data distribution and they cannot assign weights adaptively.
On the other hand, Chen et al. \cite{chen2023meta} determine the weights based on the similarity between training and test samples, but the choice of similarity functions is user-defined and not automated.
Furthermore, for the prediction in a long time series, the testing phase is computationally expensive as each new test data requires a separate process to update the penalty weights in the loss function.

\subsection{Fine-tuning}
With the advent of the large foundational model era, fine-tuning has emerged as a valuable technique to refocus the models to address additional specific tasks and to achieve robust and stable performance \cite{houlsby2019parameter, zhang2023instruction}. 
Foundational models are built by initially training a model on an extensive dataset to learn the comprehensive foundational knowledge and achieve a baseline performance on standard tasks.
Subsequently, the trained models undergo fine-tuning to tailor them to specific tasks, which involves utilizing a limited number of exclusive samples \cite{lester2021power, azad2023foundational}. 
Inspired by the success of fine-tuning, we propose a similar approach for generalizing foundational models to perform well on rare extreme events. 
In our work, foundational models are trained with the entire training set consisting of both normal and extreme events; fine-tuning is exclusively done with only rare extreme events. 
To the best of our knowledge, we have not seen work that employs \emph{fine-tuning} in the context of extreme event prediction.
\section{Problem Formulation}
\label{sec:problem_form}

For a given time series, let $\mathbf{Z}_t$ be an observation at time $t$. For generality, if $d$ different observations are collected at each time point, we assume that $\mathbf{Z}_t=\{z_1(t), \ldots, z_d(t)\} \in \mathbf{R}^d$ is a vector of dimension $d$.
In general, the ``target'' variable(s) to be predicted, $\mathbf{Z}_{t+\Delta{t}} \in \mathbf{R}^{d^{*}} (d^{*} \leq d)$, is selected from one or more of the dimensions in the observation vector.

\begin{definition}[\textbf{Time Series Prediction}]
Given a sequence of $\alpha$ time points from the past (called ``look-back window'') to predict the target variable(s) for $\beta$ time points in the future (called ``prediction window''). It can be described as:
$$[\mathbf{Z}_{t-(\alpha-1)}, \ldots, \mathbf{Z}_{t}] \xrightarrow[]{\mathcal{F(\cdot)}} [\mathbf{Z}_{t+1}, \ldots, \mathbf{
Z}_{t+\beta}].$$
\end{definition}

\begin{definition}[\textbf{Extreme Events}]
Extreme events occur when one or more observation values within a window (either look-back or prediction) cross a specific threshold, $\xi$.
In our work, we choose that threshold to be the $95^{th}$ percentile value within some set of observations. 
\end{definition}

\begin{definition}[\textbf{Long-tailed Distributions}]
Long-tailed data distributions are characterized by dominant samples with rare samples in the tail of the distribution. 
\end{definition}
Extreme values in many time series occupy the long-tailed zone.
When the extreme samples are a small part of the data, but with enormous impact, then the resulting data imbalance needs to be addressed in the models.
We refer to the data samples with (without, resp.) extreme events as extreme (normal, resp.) samples. 
Let $(x, y)$ be a input-target pair where $x \in \mathbf{R}^{\alpha \times d}$ and $y \in \mathbf{R}^{\beta \times d^{*}}$ refer to input and output time series. We have $\mathcal{D}:=\{(x_i, y_i)\}_{i=1}^{N}$ be the training set that includes pairs of both extreme samples $\mathcal{D}_{extre}:=\{(x_i, y_i)\}_{i=1}^{P}$ and normal samples $\mathcal{D}_{norm}:=\{(x_i, y_i)\}_{i=1}^{Q}$, where $P \ll Q < N$. 
The imbalanced training set causes the long-tailed distribution (Figure \ref{fig:form_a}).
We assume that there is a small clean and unbiased evaluation set, $\mathcal{D}_{extre}^{e}:=\{(x_i, y_i)\}_{i=1}^{M}$, where $M \ll N$
(Figure \ref{fig:form_b}). 
Hereafter, we will use superscript $e$ to denote the evaluation set and subscript $i$ to denote the $i^{th}$ data.
Our task is to train a DL model that can generalize well on the rare extreme samples in the evaluation set, without compromising performance on normal samples. 
We reiterate that our skewed training sets have a majority of normal samples and a minority of extreme samples and that we set aside an evaluation set with \textbf{ONLY} extreme samples.
\begin{figure}[ht!]
\centering
    \begin{subfigure}[b]{0.21\textwidth} 
        \includegraphics[width=\textwidth]{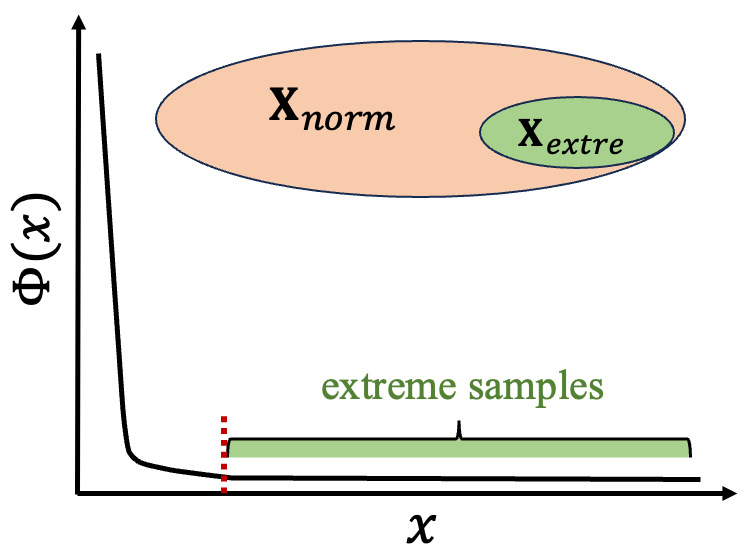}
        \caption{Training set.}
        \label{fig:form_a}
    \end{subfigure}
    \quad   
    \begin{subfigure}[b]{0.21\textwidth} 
        \includegraphics[width=\textwidth]{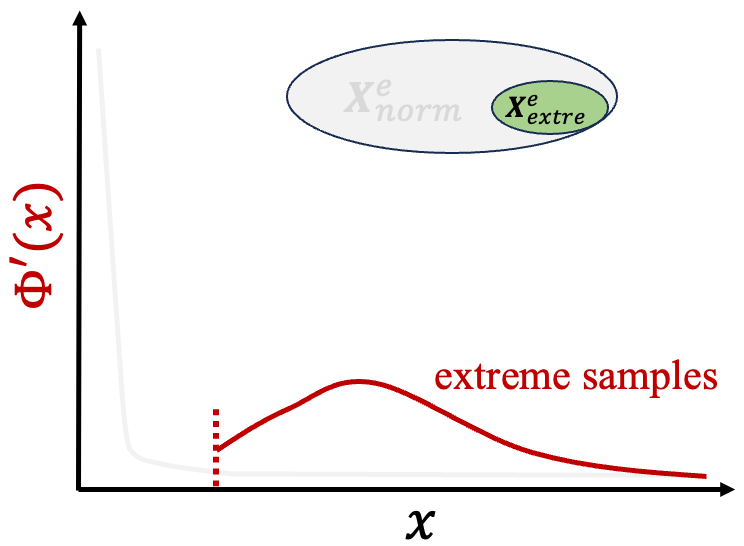}
        \caption{Evaluation set.}
        \label{fig:form_b}
    \end{subfigure}
\caption{Illustration of data distribution. $\Phi$ and $\Phi'$ are the probability distribution functions of the training and evaluation set. The dashed line refers to the threshold to split extreme and normal samples. The oval sizes represent the set sizes.}
\label{fig:formulation}
\end{figure}

\begin{figure*}[ht]
\centering
\includegraphics[width=1.83\columnwidth]{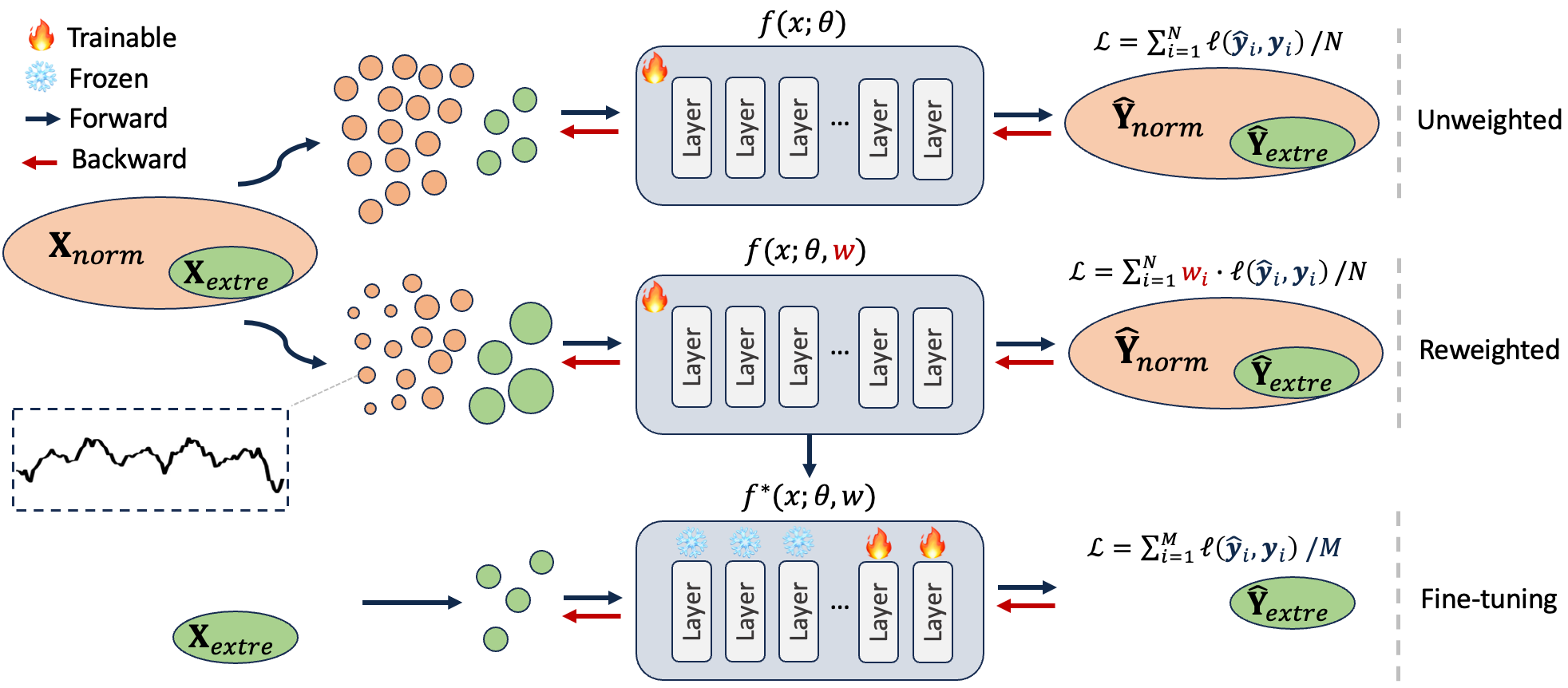} 
\caption{Training process of the unweighted framework, the reweighting approach, and the fine-tuning method. The ovals represent the sample spaces;
the small circles represent individual inputs, while their sizes denote their weights; $\textbf{y}_i$ and $\hat{\textbf{y}}_i$ are the ground truth and prediction values, respectively. Trainable models are marked with the ``fire'' symbol in the upper left corner; individual layers are marked as trainable or frozen during fine-tuning.}

\label{fig:method}
\end{figure*}

\section{Methodology}
\label{sec:method}
We formulate our approach with two steps. 
First, given both normal and extreme samples, we employ a \emph{reweighting} strategy to encourage the models to focus training on the minority extreme samples and prevent the vast number of normal samples from biasing the predictor. 
In the second step, we utilize a \emph{fine-tuning} strategy to \textbf{further} adapt the models to these minority samples by retraining them exclusively on \textbf{ONLY} rare extreme events.
See Figure \ref{fig:method} for details.
In what follows, we will discuss the three methods for the design of the penalty weights used in the loss function.

\subsection{Reweighting}
\label{sec:reweight}
The traditional training manner attempts to minimize the expected loss over a data set of size $N$: $\mathcal{L} = \frac{1}{N} \sum_{i=1}^N \ell(\hat{y}_i, y_i)$, where ${y_i}$ and $\hat{y}_{i}=f(x_i; \theta)$ are the ground-truth and the predicted values and $\ell(\hat{y}_i, y_i)$ measures the prediction errors. The function $f(x_i; \theta)$ is the predictive model with parameters $\theta$.
The loss function mentioned above equally weights the prediction errors for all samples.
In reweighting, the prediction errors are weighted differentially to emphasize those on specific subsets of the training data. The weighted loss function is $\mathcal{L}(\theta, w) = \frac{1}{N} \sum_{i=1}^N w_i \cdot \ell(\hat{y}_i, y_i)$. Because the weights will impact the model parameter $\theta$, the model is trained to seek:
\begin{equation}
\label{eq:weighted_loss}
    \theta^{*} = \arg \underset{\theta}{\min} \frac{1}{N} \sum_{i=1}^N w_i \cdot \ell_{i}(\theta),
\end{equation}
where $w_i$ weights the prediction error for the $i$-th training sample and $\ell_{i}(\theta)=\ell(\hat{y}_i, y_i)$.

In the following, we present three implementing approaches for the reweighting strategy. 
The first two methods are based on heuristics and rely on prior knowledge of the distribution of the training data, while the third method attempts to learn the optimal weights automatically, guided by a separate and unbiased dataset consisting of only extreme samples.

\subsubsection{Inverse Proportional Function}
The initial approach involves creating a frequency histogram of all training samples and determining the weights for the prediction error of each sample based on the inverse frequency of its group. 
We use $B = 20$ bins in our experiments, with the bin sizes denoted by $\{n_j: j=1, \ldots, B\}$.
The weights for the errors on samples from each bin are set to the inverse of its size, thus making $w_j = \frac{1}{n_j}$ for $j=1, \ldots, B$.

\subsubsection{Extreme Value Theory}
Extreme Value Theory (EVT) takes a further step in studying the extreme values located in the tail zone \cite{malevergne2006power}.
The cumulative distribution function (CDF) of $Z\sim GPD(\mu, \sigma, \xi)$ \cite{norberg1998p} is defined by Eq. (\ref{eq:gpd}):
\begin{equation}
    G_{\xi}(z) = 
    \begin{cases} 
        1-\exp\left(e^{-z}\right), & \xi = 0 \\
        1-\left((1 + \xi z)^{-\frac{1}{\xi}}\right), & \xi \neq 0 
    \end{cases}
\label{eq:gpd}
\end{equation}
The values exceeding a threshold $\mu$ can be approximated by the generalized Pareto distribution (GPD) if the threshold $\mu$ is sufficiently large \cite{haan2006extreme,chen2023meta}. 
Supposing $T$ random variables $y_1, \ldots, y_T$ are \texttt{i.i.d} sampled from distribution $F_Y$,
we leverage the GPD to estimate the extreme data $F(y)$ \cite{chen2023meta,ding2019modeling,malevergne2006power} in the long-tailed zone as follows. 
\begin{align}
    1 - F(y) &\approx (1 - F(\xi))(1-G_{\xi}(\frac{y-\mu}{\sigma})), & y > \mu \\
             & = (1 - F(\xi))(1+\frac{\xi(y-\mu)}{\sigma})^{-\frac{1}{\xi}}, & y > \mu
\end{align}
where $\mu$ is the location parameter, $\sigma$ is the scale of the distribution which is analogous to the standard deviation in a normal distribution, and $\xi$ is the extreme value index, determining the heaviness of the tail of the distribution.
 
Finally, the weights for the prediction errors on extreme samples are set to the inverse of their probability:
\begin{equation}
    w_i = 
    \begin{cases} 
        \frac{1}{1-F(y_i)}, & y_i \geq \mu \\
        c, & y_i < \mu   
    \end{cases}
\end{equation}
where $c$ is a small weight assigned to the error on each normal sample.
\subsubsection{Meta Learning}
The preceding two strategies calculate the penalty weights by leveraging prior knowledge of the data distribution.
In our approach, we consider the weights as hyperparameters that can influence the model's parameter $\theta$, as shown in Eq. (\ref{eq:weighted_loss}).
Therefore, we utilize meta-learning to dynamically learn the optimal ones that can minimize the loss function of the exclusive evaluation set: 
\begin{equation}
\label{eq:weight_opt}
    w^* = \arg \underset{w}{\min} \frac{1}{M} \sum_{j=1}^{M} \ell_{j}(\theta^*(w)),
\end{equation}
where $M$ is the size of the evaluation set that includes only extreme samples.

The specific implementation process is described as follows. See the schematic in Figure \ref{fig:meta_reweight} and the pseudo-code in Algorithm \ref{alg:meta_training}. 
For each training iteration, we inspect the descent direction of a batch of training examples locally on the training loss surface and reweight them according to their similarity to the descent direction of the evaluation loss surface. 
At every step $t$ of training, a mini-batch of training examples $\mathcal{D}^{batch}:=\{(x_i, y_i)\}_{i=1}^{n}$ is sampled, where $n$ is the mini-batch size, and $n \ll N$. 
We first initialize the weight, $w_i$, to the prediction error on that training sample within the mini-batch, and use stochastic gradient descent (SGD) to optimize a weighted loss function $\ell_{i,w}(\theta) = w_{i} \cdot \ell_i(\theta)$ with a learning rate $\phi$ (see step 2 in Figure \ref{fig:meta_reweight}):
\begin{equation}
\label{eq:SGD}
    \hat{\theta}_{t+1} = \theta_t - \phi \nabla \left( \frac{1}{n} \sum_{i=1}^{n} w_{i,t} \cdot \ell_i(\theta_t) \right).
\end{equation}
After obtaining the updated model parameters, $\hat{\theta}_{t+1}(w)$, we evaluate them on a mini-batch of evaluation samples, $\mathcal{D}_{extre}^{e}$ of size $m$, with $m \ll M $. See step 3 in Figure \ref{fig:meta_reweight}.
Next, we take a single gradient descent step on a mini-batch of evaluation samples concerning $w_{t}$, and rectify a non-negative weight:
\begin{equation} 
\label{eq:update_weight}
    \hat{w}_{i,t+1} = w_{i,t} - \eta \nabla \left( \frac{1}{m} \sum_{j=1}^{m}\ell^{e}_j(\hat{\theta}_{t+1}(w_t)) \Big|_{w_{i,t}} \right),
\end{equation}
\begin{equation}
    \tilde{w}_{i,t+1} = \max(\hat{w}_{i,t+1}, 0).
\end{equation}
where $\eta$ is the descent step size on weight $w$.
To match the original training step size, we consider normalizing the weights of all examples in a training batch:
\begin{equation}
    w_{i,t+1} = \frac{\tilde{w}_{i,t+1}}{\sum_{j} \tilde{w}_{j,t+1} + \delta\left(\sum_{j} \tilde{w}_{j,t+1}\right)},
\end{equation}
where $\delta(\cdot)$ prevents the degenerate case when all weights are 0 in a mini-batch, i.e. $\delta(a) = 1$ if $a = 0$, and equals $0$ otherwise. 

Then the model parameters $\theta_{t}$ are adjusted to $\theta_{t+1}$ according to the updated penalty weights of the current batch such that so that $\theta_{t+1}$ can consider the meta information from the evaluation set:
\begin{equation}
\label{eq:update_theta}
    \theta_{t+1} = \theta_t - \phi \nabla \left( \frac{1}{n} \sum_{i=1}^{n} w_{i, t+1} \cdot \ell_i(\theta_t) \right).
\end{equation}
\begin{figure}[ht]
\centering
    \includegraphics[width=0.9\columnwidth]{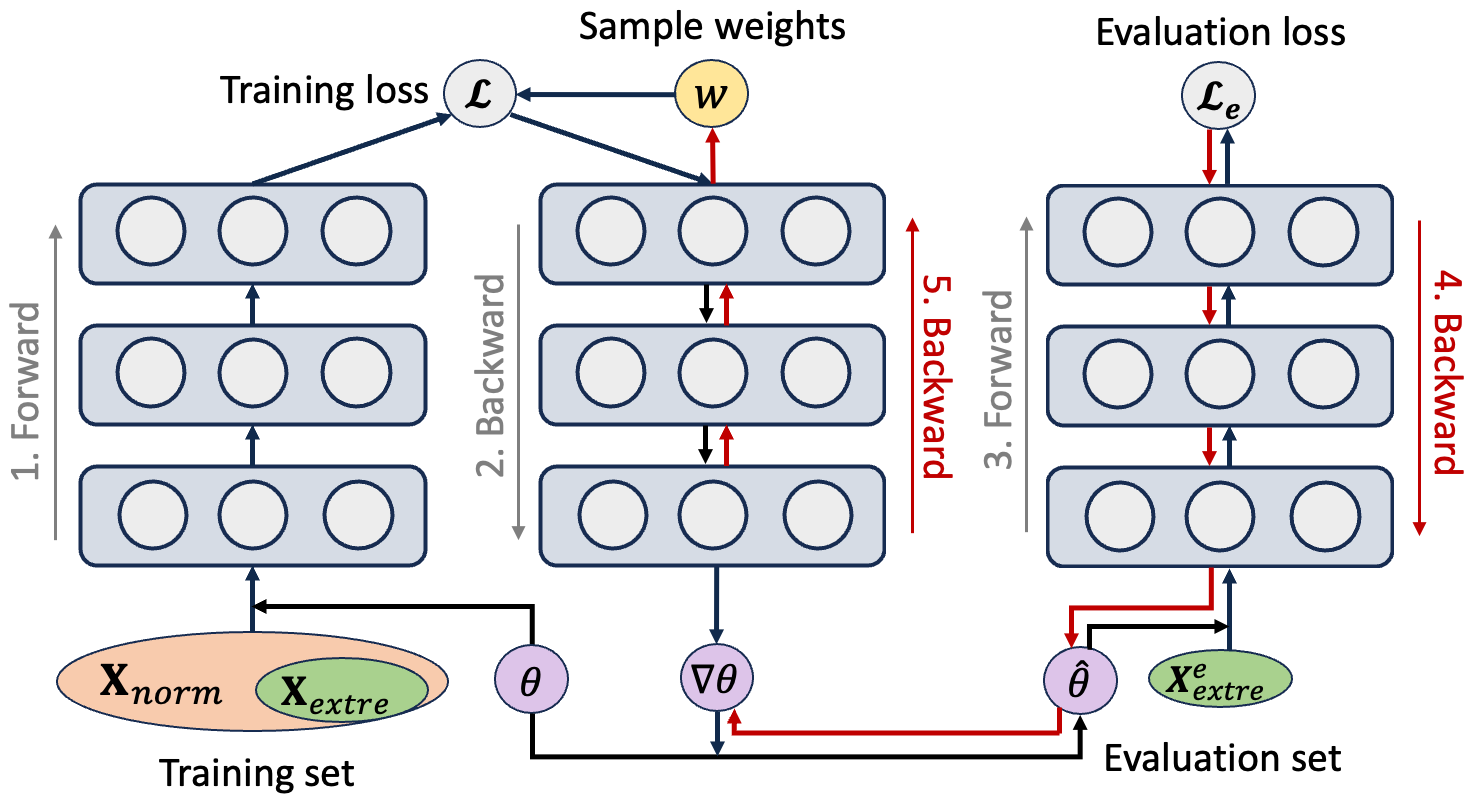} 
\caption{A schematic of the meta-learning-based reweighting method.} 
\label{fig:meta_reweight}
\end{figure}
\begin{algorithm}[ht]
\caption{Pseudo-code of meta-learning for reweighting}
\small
\label{alg:meta_training}
\begin{algorithmic}[1] 
\STATE \textbf{Input}: training and evaluation set: $\mathcal{D}_{train}, \mathcal{D}_{eval}$ \\ 
\STATE \textbf{Parameter}: batch size: $n, m$, iterations: $T$  \\
\FOR{$t = 0, \ldots, T-1$}
    \STATE $\{X_{train}, y_{train}\} \gets \text{SampleMiniBatch}(D_{train}, n)$
    \STATE $\{X_{eval}, y_{eval}\} \gets \text{SampleMiniBatch}(D_{eval}, m)$
    \STATE // \textcolor{blue}{\texttt{forward and backward on training set}}
    \STATE $\hat{y}_{train} \gets \text{Forward}(X_{train}, y_{train}, \theta_t)$
    \STATE $w_{t} \gets 0; \ell_{train} \gets \frac{1}{n} \sum_{i=1}^n w_{i,t} \cdot \ell(\hat{y}_{{train},i}, y_{{train},i})$
    \STATE $\nabla_{\theta_t} \gets \text{BackwardAD}(\ell_{train}, \theta_t)$
    \STATE $\hat{\theta}_{t+1} \gets \theta_t - \phi \nabla_{\theta_t}$
    \STATE // \textcolor{blue}{\texttt{forward and backward on evaluation set}}
    \STATE $\hat{y}_{eval} \gets \text{Forward}(X_{eval}, y_{eval}, \hat{\theta}_{t+1})$
    \STATE $\ell_{eval} \gets \frac{1}{m} \sum_{j=1}^m \ell(\hat{y}_{eval,j}, y_{eval,j})$
    \STATE $\nabla_{w_t} \gets \text{BackwardAD}(\ell_{eval}, w_t)$
    \STATE // \textcolor{blue}{\texttt{update penalty weights in loss function}}
    \STATE $\hat{w}_{t+1} \gets \hat{w}_{t} - \beta \nabla_{w_t}$
    \STATE $\tilde{w}_{t+1} \gets \max(\hat{w}_{t+1}, 0); w_{t+1} \gets \frac{\tilde{w}_{t+1}}{\sum_{j=1}^m \tilde{w}_{j,t+1} + \delta(\sum_{j} \tilde{w}_{j,t+1})}$
    \STATE $\hat{\ell}_{train} \gets \frac{1}{n} \sum_{i=1}^n w_{i,t+1} \cdot \ell(\hat{y}_{train,i}, y_{train,i})$
    \STATE $\nabla_{\theta_t} \gets \text{BackwardAD}(\hat{\ell}_{train}, \theta_t)$
    \STATE // \textcolor{blue}{\texttt{update model parameters}}
    \STATE $\theta_{t+1} \gets \theta_t - \phi \nabla_{\theta_t}$
\ENDFOR
\STATE \textbf{return} well-trained model $f^{*}(\theta)$, optimal weights $w^{*}$
\end{algorithmic}
\end{algorithm}

\subsubsection{Theoretical convergence analysis of meta-learning reweighting} 
\label{sec:lemma}
It is necessary to establish a convergence analysis of our meta-learning-based reweighting method since it involves the optimization of bi-level objectives (Eqs. \ref{eq:weighted_loss}, \ref{eq:weight_opt}).
In this context, we theoretically demonstrate that our method converges to the critical point of the evaluation loss function under certain mild conditions. 
In this context, the following lemma guarantees convergence of the evaluation loss.
\begin{definition}[\textbf{$\sigma$-bounded gradients \cite{garrigos2023handbook}}]
    $f(x)$ has $\sigma$-bounded gradients if $\lVert \nabla f(x) \rVert \leq \sigma$ for all $x \in \mathbb{R}^d$.
\end{definition}
%
%
%

\begin{lemma}
    \textit{Suppose the evaluation loss function is Lipschitz-smooth with constant $L$, and the train loss function $\ell_i$ of training data $x_i$ has $\sigma$-bounded gradients. Let the learning rate $\phi$ satisfy $\phi \leq \frac{2n}{L \sigma^2}$, where $n$ is the training batch size. Following our algorithm, the evaluation loss always monotonically decreases for any training batches,}
\begin{equation}
\label{eq:lipschitz}
    \mathcal{L}(\theta_{t+1}) \leq \mathcal{L}(\theta_t),
\end{equation}
where \( \mathcal{L}(\theta) \) is the total evaluation loss
\begin{equation}
    \mathcal{L}(\theta) = \frac{1}{M} \sum_{j=1}^{M} \ell_j^{e} (\theta_{t+1}(w)).
\end{equation}
The equality $\mathcal{L}(\theta_{t+1}) = \mathcal{L}(\theta_t)$ in Eq. (\ref{eq:lipschitz}) holds only when the gradient of evaluation loss becomes 0 at some time step $t$, namely $\mathbb{E}_{t}[\mathcal{L}(\theta_{t+1})] = \mathcal{L}(\theta_t)$ if and only if $\nabla\mathcal{L}(\theta_t) = 0$, where the expectation represents the possible training batches at time step $t$. 
\end{lemma}
\begin{proof}
Suppose we have another $N$ training data, $\{x_{1}, x_{2}, \ldots, x_N\}$, and the overall training loss would be $\frac{1}{N} \sum_{i=1}^{N} w_i \cdot \ell_i(\theta)$.
During training, we take a mini-batch $B$ of training data at each step and validate the model with a mini-batch $B$ of evaluation data. We set $|B| = n = m$ where $n$ and $m$ are the batch sizes of training and evaluation data, respectively. By merging Eqs. (\ref{eq:update_weight}, \ref{eq:update_theta}), we can derive:
\begin{equation}
\label{eq:theta_diff}
    \theta_{t+1} = \theta_t - \phi \frac{1}{n} \sum_{i \in B} \max \{ \nabla \mathcal{L}^{\texttt{T}} \nabla \ell_i, 0 \} \nabla \ell_i, 
\end{equation}
where $\phi_t$ is the learning rate at time-step $t$, $\max \{ \nabla \mathcal{L} \nabla \ell_i, 0 \}$ is the evaluation gradients with respect to the weights, and $\nabla \ell_i$ is the training gradients with respect to the parameters $\theta_t$. \\

\noindent Since the evaluation loss $\mathcal{L}(\theta)$ is Lipschitz-smooth \cite{bubeck2015convex} with constant $L$
\begin{equation}
    \lVert \nabla \mathcal{L}(x) - \nabla \mathcal{L}(y) \rVert \leq L \lVert x - y \rVert, \forall x, y \in \mathbb{R}^d,
\end{equation}
and consider the Taylor's Remainder Theorem \cite{poffald1990remainder}, we have:
\begin{equation}
\label{eq:loss_inequal}
    \mathcal{L}(\theta_{t+1}) \leq \mathcal{L}(\theta_t) + \nabla \mathcal{L}^{\texttt{T}} \Delta \theta + \frac{L}{2} \|\Delta\theta\|^2. 
\end{equation}
Now we need to prove $\mathcal{L}^{\texttt{T}} \Delta \theta + \frac{L}{2} \|\Delta\theta\|^2 \leq 0$. Plugging $\Delta \theta_t$ from Eq. (\ref{eq:theta_diff}) into $\mathcal{L}^{\texttt{T}} \Delta \theta$, we have
\begin{equation}
\begin{aligned}
\nabla \mathcal{L}^{\texttt{T}} \Delta \theta &= -\frac{\phi}{n} \sum_{i \in B} \max \{ \nabla \mathcal{L}^{\texttt{T}} \nabla \ell_i, 0 \} \nabla \mathcal{L}^{\texttt{T}} \nabla \ell_i,  \\
                                              &= -\frac{\phi}{n} \sum_{i \in B} \max\{\nabla \mathcal{L}^{\texttt{T}} \nabla \ell_i, 0\}^2 \leq 0 \text{ holds},
\end{aligned}
\end{equation}
and,
\begin{align}
\frac{L}{2} \|\Delta\theta\|^2  &= \frac{L}{2} \left( \frac{\phi}{n} \sum_{i \in B} \max\{\nabla \mathcal{L}^{\texttt{T}} \nabla \ell_i, 0\}\nabla \ell_i \right)^2, \\
                                &\leq \frac{L \phi^2}{2 n^2} \sum_{i \in B} \left| \max \{\nabla \mathcal{L}^{\texttt{T}} \nabla \ell_i, 0\} \nabla \ell_i \right|^2, \label{eq:inequal1} \\
                                &= \frac{L \phi^2}{2 n^2} \sum_{i \in B} \max\{\nabla \mathcal{L}^{\texttt{T}} \nabla \ell_i, 0\}^2 \|\nabla \ell_i\|^2,  \\
                                &\leq \frac{L \phi^2}{2 n^2} \sum_{i \in B} \max\{\nabla \mathcal{L}^{\texttt{T}} \nabla \ell_i, 0\}^2 \sigma^2 \label{eq:inequal2}. 
\end{align}
%
The first inequality in Eq. (\ref{eq:inequal1}) comes from the triangle inequality. The second inequality in Eq. (\ref{eq:inequal2}) holds since $\ell_i$ has $\sigma$-bounded gradients \cite{garrigos2023handbook}. If we let $\Gamma_t = \max\{\nabla \mathcal{L}^{\texttt{T}} \nabla \ell_i, 0\}^2$, then
\begin{equation}
    \mathcal{L}(\theta_{t+1}) \leq \mathcal{L}(\theta_t) - \frac{\phi}{n} \Gamma_t \left( 1 - \frac{L\phi}{2n} \sigma^2 \right). 
\end{equation}
Note that $\Gamma_t$ is non-negative, and since $\phi \leq \frac{2n}{L\sigma^2}$, it follows that $\mathcal{L}(\theta_{t+1}) \leq \mathcal{L}(\theta)$ for any $t$. 
\end{proof}

\subsection{Fine-tuning}
\label{sec:finetune}
After applying the reweighting technique, a set of weights has been computed for the prediction errors of training samples in the loss function. 
Here, we elucidate the process of adapting the trained models to achieve robust generalization for extreme samples.
In our tasks, we freeze the first several layers to maintain the original comprehensive knowledge of both the majority normal and minority extreme samples.
Subsequently, we conduct fine-tuning exclusively on the latter layers to adapt the model using only the extreme samples (see the bottom fine-tuning in Figure \ref{fig:method}). 
Given the typically limited number of training samples used during fine-tuning, we add $L2$ regularization to the remaining trainable layers as a precaution against potential overfitting.
\section{Experiments}
\subsection{Datasets}
\label{sec:dataset}
We conduct experiments on four public real-world data sets: Beijing PM2.5, Jena Climate, Spain Electrical Demand, and South Florida water management data. The summary of each dataset is shown in Table \ref{tab:datasets}.
\begin{itemize}
    \item \textbf{Beijing PM2.5 \cite{misc_beijing_pm25}.} 
    The PM2.5 index is the target variable to predict; covariates include dew, temperature, pressure, wind speed, wind direction, snow, and rain. $PM2.5 \in [0, 671] \mu g/m^3$.
    \item \textbf{Jena Climate \cite{misc_jena_climate}.} 
    Recorded by the Max Planck Institute in Jena, Germany for Biogeochemistry, this dataset consists of features such as temperature, pressure, and humidity, recorded once every 10 minutes. We use the hourly data for our experiments. Saturation vapor pressure is the target variable to predict and its values $\in [0, 62.94]$ mbar.
    \item \textbf{Spain Electricity \cite{misc_spain_electricity}.}
    This dataset contains data on electrical consumption, generation, pricing, and weather in Spain. In this dataset, we predict two target variables: electricity price $\in [\$9.33, \$116.8]$ and the load $\in [18041.0, 41015.0]$.
    \item \textbf{Florida Water \cite{shi2023deep}.} 
    It includes water levels at multiple stations, control schedules of hydraulic structures, tide and rainfall information in South Florida. Water levels are the target variables $\in \text{[-1.25, 4.05]}$ feet.
\end{itemize}
\begin{table}[ht]
\centering
\caption{Summary of Datasets}
    \begin{tabular}{lcccc}
    \toprule
    Dataset   & PM2.5      & Climate    & Electricity & Water Level\\
    \midrule
    Start     & 2010/01/01 & 2009/01/10 & 2015/01/01 & 2010/01/01 \\
    End       & 2014/12/31 & 2016/12/31 & 2018/12/31 & 2020/12/31 \\
    Interval  & 1 hour     & 1 hour     & 1 hour     & 1 hour  \\
    \#Time Point    & 43,800     & 70,129     & 35,063     & 96,432\\
    \#Feature & 11         & 14         & 26         & 19 \\ 
    \#Extreme & 2,180      & 3,507      & 1,752      & 4,715 \\ 
    \#Normal  & 41,620     & 66,622     & 33,311     & 91,717 \\ 
    E:N ratio & 1:19       & 1:19       & 1:19       & 1:19 \\ 
    \bottomrule
    \end{tabular}
\label{tab:datasets}
\end{table}

\subsection{Experiment Setting}
\label{sec:exp}
We set the length of look-back window $\alpha=72$ hours and prediction length $\beta=12$ or $24$ for time series forecasting ($\beta=24$ for the last data set while $\beta=12$ for others). 
In the cases of the first three datasets, we define extreme samples by examining the values of target variables that exceed $95^{th}$ percentile. We aim to predict these extreme events in the future $\beta$ time points.
For the last data set, we select extreme samples by calculating the covariate precipitation that is over $95^{th}$ percentile and predict the water levels in the river since heavy rainfall events have much impact.

\subsection{Training and Evaluation}
Each data set has been divided in chronological order with 70\% for training, 15\% for validation\footnote{The validation set with only extreme samples serve as the evaluation set in Figures \ref{fig:formulation} and \ref{fig:meta_reweight}.}, and 15\% for testing.
To prove the efficacy of reweighting and fine-tuning strategies, we choose the simple multi-layer perceptron (MLP) as the backbone. The architecture comprises 8 hidden layers, with each layer being a fully connected layer consisting of $128, 128, 64, 64, 32, 32, 16,$ and $16$ neurons, respectively. To potentially regularize the model, each hidden layer is followed by a Dropout layer, and we considered dropout factors from the set {0, 0.1, 0.2} as candidates. In total, there are 16 layers between \texttt{Input} and \texttt{Output} layer.
We apply Max-Min normalization to scale the input data within the range [0, 1], mitigating potential biases stemming from varying scales. 
The learning rate is $1e-4$, the batch size is $500$, and $1000$ and $500$ epochs are used for reweighting and fine-tuning.
We utilize early stopping with $50$ patience steps and regularization $L_2=1e-6$ to counteract overfitting. 
After obtaining the well-trained models, we test them on the extreme samples from the test set using mean absolute errors (MAEs) and root mean square errors (RMSEs). 
All experiments are performed with one NVIDIA A100 GPU with 24G memory.

\subsection{Baselines}
We consider baselines including unweighted models, \texttt{LSTM}, \texttt{Transformer} and \texttt{Informer}, and some existing weighted models using the inverse proportional function (\texttt{IPF}), extreme value theory (\texttt{EVT}), and \texttt{NEC+}.
\begin{itemize}
    \item \texttt{TCN} \cite{van2016wavenet}.
    A model that uses a hierarchy of temporal convolutional networks (TCNs) for time series forecasting.
    \item \texttt{LSTM} \cite{graves2012long}. 
    A variant of recurrent neural networks (RNN) aims at dealing with the vanishing gradient problem present in traditional RNNs.
    \item \texttt{Transformer} \cite{vaswani2017attention}. 
    An \emph{attention}-based model that can be used for time series forecasting.
    \item \texttt{Autoformer} \cite{wu2021autoformer}.
    An \emph{attention}-based model with the auto-correlation mechanism for long-term time series prediction.
    \item \texttt{FEDformer} \cite{zhou2022fedformer}.
    A frequency-enhanced decomposed Transformer architecture with seasonal-trend decomposition for time series forecasting.
    \item \texttt{NEC+} \cite{li2023extreme}.
    A reweighted benchmark for extreme event prediction by assigning a probability as the weight for extreme and normal classes.
    \item \texttt{IPF} \cite{patterson2013parasite}. 
    A reweighted method to deal with imbalanced data determines the weights based on the frequency histogram of training samples.
    \item \texttt{EVT} \cite{ding2019modeling}. 
    A reweighted method that determines the weights based on the extreme value theory.
\end{itemize}

\subsection{Reweighting}
Table \ref{tab:reweight} reports the results across four datasets on five cases.
The reweighting methods implemented in our work demonstrate a statistically significant and consistent improvement over the benchmarks.
The meta-learning-based reweighting surpasses the other two methods (\texttt{IPF} and \texttt{EVT}) that determine weights using prior knowledge of the data distribution. This confirms the significant advantages of seeking optimal weights in an automated manner.
Moreover, the heuristic reweighting techniques that employ the inverse proportional function (IPF) and extreme value theory (EVT) perform closely to each other.
We provide a visualization of 50 samples at time $t+1$ in Figure \ref{fig:pred_true_vis}.

Additionally, we conduct an ablation study on the unweighted \texttt{MLP} model by including only normal or extreme samples during training. The results are shown in Table \ref{tab:ablation}, it is worth noting that \texttt{Unweighted\_Both} performs better than the other two methods, \texttt{Unweighted\_Normal} and \texttt{Unweighted\_Extreme}. 
This observation shows training solely on normal samples struggles to adapt to dynamic distribution changes from extreme samples during testing, while exclusive training on extremes risks overfitting due to limited sample quantity.
This emphasizes the significance of incorporating both normal and extreme samples.
\begin{table*}[ht]
\centering
\caption{Experimental results on \textcolor{red}{extreme} samples in the test set. The names starting with ``Reweight'' represent models implemented in our work. $\Delta$ denotes the relative improvement of our best reweighting method in bold$^{*}$ compared with the best benchmark with \underline{underline}. $^{*}$: p-value $< 0.05$.}
\begin{tabular}{l|cc|cc|cc|cc|cc}
\toprule
\multirow{3}{*}{\textbf{Model}} &  \multicolumn{2}{c|}{\textbf{Ele-Price}} & \multicolumn{2}{c|}{\textbf{Pressure}} & \multicolumn{2}{c|}{\textbf{PM25}} & \multicolumn{2}{c|}{\textbf{Ele-Load}} & \multicolumn{2}{c}{\textbf{Water Level}}\\
\cmidrule(lr){2-3} \cmidrule(lr){4-5} \cmidrule(lr){6-7} \cmidrule(lr){8-9} \cmidrule(lr){10-11}
  & \multicolumn{1}{c}{\textbf{MAE}} & \multicolumn{1}{c|}{\textbf{RMSE}} & {\textbf{MAE}} & {\textbf{RMSE}} & {\textbf{MAE}} & {\textbf{RMSE}} & {\textbf{MAE}} & {\textbf{RMSE}} & {\textbf{MAE}} & {\textbf{RMSE}} \\
\midrule
TCN         & 3.84              & 5.21              & \underline{2.96}  & 4.81              & 38.58             & 58.15             & 1747.50             & 2465.90             & 0.148             & 0.188 \\
NEC+        & 4.02              & 5.25              & 3.52              & 4.89              & 44.54             & 63.10             & 1698.59             & 2059.74             & 0.141             & 0.181\\
LSTM        & 4.20              & 5.36              & 2.99              & \underline{4.06}  & 42.71             & 61.71             & 1653.50             & 2144.55             & \underline{0.115} & \underline{0.151} \\
Transformer & \underline{3.71}  & \underline{4.83}  & 2.98              & 4.21              & 38.62             & 57.81             & \underline{1386.93} & \underline{1806.63} & 0.116             & 0.159 \\
Autoformer  & 4.85              & 6.29              & 3.74              & 5.29              & 55.56             & 57.91             & 1610.72             & 2276.99             & 0.164             & 0.213 \\
FEDformer   & 3.99              & 5.21              & 3.72              & 5.22              & \underline{37.68} & \underline{55.71} & 1644.16             & 2241.10             & 0.153             & 0.197 \\
\midrule
Reweight\_\texttt{IPF}  & 3.54           & 4.64              & 2.89              & 4.01              & 36.53             & 54.69             & 1357.58             & 1681.05             & 0.108             & 0.154\\
Reweight\_\texttt{EVT}  & 3.57           & 4.67              & 2.91              & 4.04              & 36.81             & 54.31             & 1304.71             & 1644.91             & 0.112             & 0.158\\
Reweight\_\texttt{META} & \textbf{3.52$^{*}$} & \textbf{4.62$^{*}$} & \textbf{2.75$^{*}$} & \textbf{3.89$^{*}$} & \textbf{35.18$^{*}$} & \textbf{53.55$^{*}$} & \textbf{1129.36$^{*}$} & \textbf{1434.92$^{*}$} & \textbf{0.106$^{*}$} & \textbf{0.142$^{*}$}\\
\midrule
\textcolor{brown}{Improvement $\Delta$ \%} & \textcolor{brown}{7.85\%} & \textcolor{brown}{6.10\%} & \textcolor{brown}{7.09\%} & \textcolor{brown}{4.18\%} & \textcolor{brown}{6.63\%} & \textcolor{brown}{3.87\%} & \textcolor{brown}{18.57\%} & \textcolor{brown}{20.57\%} & \textcolor{brown}{7.82\%} & \textcolor{brown}{5.96\%}\\
\bottomrule
\end{tabular}
\label{tab:reweight}
\end{table*}
\begin{table*}[ht]
\centering
\caption{The boosting performance of fine-tuning on the basic reweighting method in Table \ref{tab:reweight}. $\Delta$ denotes the relative improvement of our fine-tuning method compared to the previous unweighted/reweighted method without fine-tuning.}
\begin{tabular}{l|cc|cc|cc|cc|cc}
\toprule
\multirow{3}{*}{\textbf{Model}} &  \multicolumn{2}{c|}{\textbf{Ele-Price}} & \multicolumn{2}{c|}{\textbf{Pressure}} & \multicolumn{2}{c|}{\textbf{PM25}} & \multicolumn{2}{c|}{\textbf{Ele-Load}} & \multicolumn{2}{c}{\textbf{Water Level}}\\
\cmidrule(lr){2-3} \cmidrule(lr){4-5} \cmidrule(lr){6-7} \cmidrule(lr){8-9} \cmidrule(lr){10-11}
  & \multicolumn{1}{c}{\textbf{MAE}} & \multicolumn{1}{c|}{\textbf{RMSE}} & {\textbf{MAE}} & {\textbf{RMSE}} & {\textbf{MAE}} & {\textbf{RMSE}} & {\textbf{MAE}} & {\textbf{RMSE}}  & {\textbf{MAE}} & {\textbf{RMSE}}\\
\midrule
Reweight\_IPF & 3.54 & 4.64 & 2.89 & 4.01 & 36.53 & 54.69 & 1357.58 & 1681.05 & 0.108 & 0.154\\
Reweight\_IPF\_Finetune & 3.52 & 4.59 & 2.75 & 3.79 & 35.46 & 53.35 & 1288.52 & 1617.46 & 0.103 & 0.145\\
\textcolor{brown}{Improvement $\Delta$ \%} & \textcolor{brown}{0.56\%} & \textcolor{brown}{1.08\%} & \textcolor{brown}{4.84\%} & \textcolor{brown}{5.49\%} & \textcolor{brown}{2.93\%} & \textcolor{brown}{2.45\%} & \textcolor{brown}{5.09\%} & \textcolor{brown}{3.78\%} & \textcolor{brown}{4.63\%} & \textcolor{brown}{5.84\%} \\
\midrule

Reweight\_EVT & 3.57 & 4.67 & 2.91 & 4.04 & 36.81 & 54.31 & 1304.71 & 1644.91 & 0.112 & 0.158\\
Reweight\_EVT\_Finetune & 3.51 & 4.53 & 2.73 & 3.89 & 36.19 & 53.29 & 1200.23 & 1532.90 & 0.104 & 0.148 \\
\textcolor{brown}{Improvement $\Delta$ \%} & \textcolor{brown}{1.68\%} & \textcolor{brown}{3.00\%} & \textcolor{brown}{6.19\%} & \textcolor{brown}{3.71\%} & \textcolor{brown}{1.68\%} & \textcolor{brown}{1.88\%} & \textcolor{brown}{8.01\%} & \textcolor{brown}{6.81\%} & \textcolor{brown}{7.14\%} & \textcolor{brown}{6.33\%} \\
\midrule

Reweight\_META & 3.52 & 4.62 & 2.75 & 3.89 & 35.18 & 53.55 & 1129.36 & 1434.92 & 0.106 & 0.142 \\
Reweight\_META\_Finetune & 3.50 & 4.57 & 2.66 & 3.76 & 34.54 & 52.80 & 1141.56 & 1464.46 & 0.103 & 0.139 \\
\textcolor{brown}{Improvement $\Delta$ \%} & \textcolor{brown}{0.57\%} & \textcolor{brown}{1.08\%} & \textcolor{brown}{3.27\%} & \textcolor{brown}{3.34\%} & \textcolor{brown}{1.82\%} & \textcolor{brown}{1.40\%} & \textcolor{brown}{-1.08\%} & \textcolor{brown}{-2.06\%} & \textcolor{brown}{6.60\%} & \textcolor{brown}{2.11\%} \\
\bottomrule
\end{tabular}
\label{tab:finetune}
\end{table*}
\begin{table*}[ht]
\centering
\caption{Abalation study by exclusively training on ``Normal", ``Extreme'' samples. Below is the experimental results on \textcolor{red}{extreme} samples in the test set without reweighting. ``Normal", ``Extreme'' and ``Both" refer to the unweighted methods trained using only normal samples, extreme samples, and both. The best is marked in bold.}
\begin{tabular}{l|cc|cc|cc|cc|cc}
\toprule
\multirow{3}{*}{\textbf{Model}} &  \multicolumn{2}{c|}{\textbf{Ele-Price}} & \multicolumn{2}{c|}{\textbf{Pressure}} & \multicolumn{2}{c|}{\textbf{PM25}} & \multicolumn{2}{c|}{\textbf{Ele-Load}} & \multicolumn{2}{c}{\textbf{Water Level}}\\
\cmidrule(lr){2-3} \cmidrule(lr){4-5} \cmidrule(lr){6-7} \cmidrule(lr){8-9} \cmidrule(lr){10-11}
  & \multicolumn{1}{c}{\textbf{MAE}} & \multicolumn{1}{c|}{\textbf{RMSE}} & {\textbf{MAE}} & {\textbf{RMSE}} & {\textbf{MAE}} & {\textbf{RMSE}} & {\textbf{MAE}} & {\textbf{RMSE}} & {\textbf{MAE}} & {\textbf{RMSE}} \\
\midrule
MLP\_Unweight\_ Normal & 4.02 & 5.20 & 3.37 & 4.48 & 38.84 & \textbf{57.91} & 1700.15 & 2170.72 & 0.191 & 0.259 \\
MLP\_Unweight\_Extreme & 4.95 & 6.25 & 3.26 & 4.41 & 47.5 & 65.75 & 2054.33 & 2540.56 & 0.173 & 0.229 \\
MLP\_Unweight\_Both & \textbf{3.82} & \textbf{4.92} & \textbf{3.11} & \textbf{4.17} & \textbf{37.76} & 58.87 & \textbf{1519.34} & \textbf{1890.17} & \textbf{0.127} & \textbf{0.175} \\
\bottomrule
\end{tabular}
\label{tab:ablation}
\end{table*}
\begin{figure}[htbp]
    \centering
    \begin{subfigure}[b]{\columnwidth}
        \centering
        \includegraphics[width=0.8\linewidth]{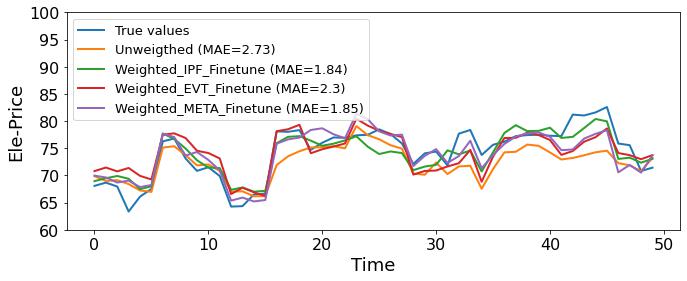}
        \caption{Ele-price}
    \end{subfigure}
    \vfill
    \begin{subfigure}[b]{\columnwidth}
        \centering
        \includegraphics[width=0.8\linewidth]{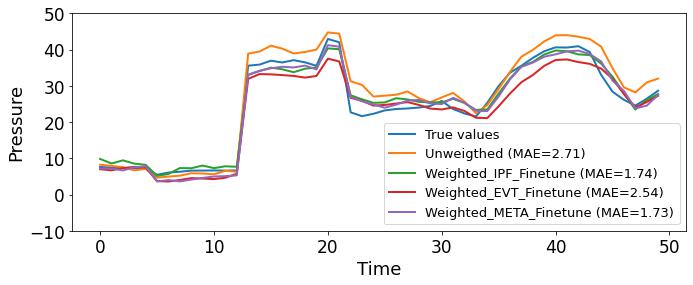}
        \caption{Pressure}
    \end{subfigure}
    \vfill
    \begin{subfigure}[b]{\columnwidth}
        \centering
        \includegraphics[width=0.8\linewidth]{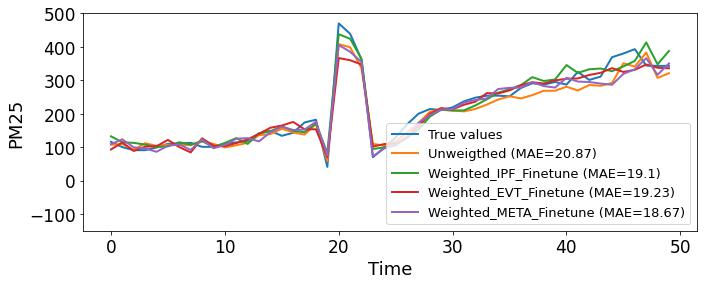}
        \caption{PM25}
    \end{subfigure}
    \vfill
    \begin{subfigure}[b]{\columnwidth}
        \centering
        \includegraphics[width=0.8\linewidth]{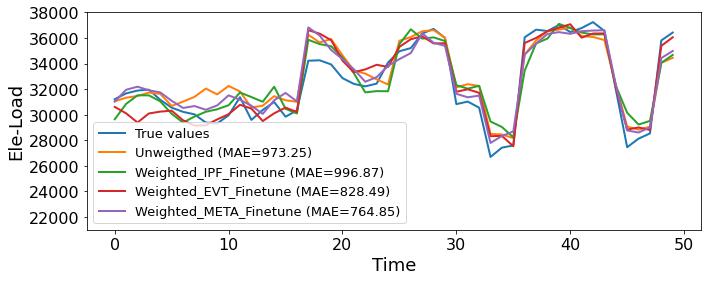}
        \caption{Ele-load}
    \end{subfigure}
    \vfill
    \begin{subfigure}[b]{\columnwidth}
        \centering
        \includegraphics[width=0.8\linewidth]{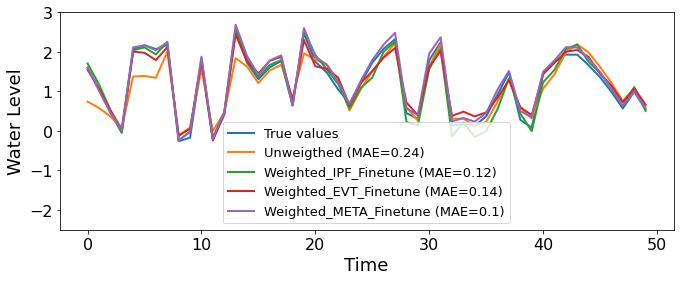}
        \caption{Water Level}
    \end{subfigure}
\caption{Visualization of truth and prediction. ``Unweighted'' is the baseline model without reweighting and fine-tuning, while the last three are with reweighting and fine-tuning.}
\vspace{-3mm}
\label{fig:pred_true_vis}
\end{figure}
\subsection{Fine-tuning}
To illustrate the boosting efficacy of the fine-tuning strategy, we fine-tune the previously reweighted model by re-training them on only rare extreme events. 
Table \ref{tab:finetune} contrasts the efficacy of models with and without fine-tuning across various datasets. We can observe fine-tuning strategy tends to further elevate the performance (refer to rows 6-11) of two heuristic reweighting methods, which underscores the value of fine-tuning and suggests that heuristic reweighting may have potential areas for improvement.
Conversely, applying fine-tuning to meta-learning-based reweighting results in minimal or even adverse effects, as seen with the \texttt{Ele-Load} data set in the final row, implying that meta-learning-based reweighting may already be at or near optimal efficacy.

Overall, while fine-tuning generally leads to improvements, the extent of its impact is influenced by both the method and the dataset in question.
It is worth noting that while the enhancement achieved through fine-tuning may appear modest, it holds significant value as it serves to \textbf{further} augment the already effective reweighting methods.
\begin{figure*}[ht]
\centering
    \begin{subfigure}[b]{0.38\columnwidth}
        \includegraphics[scale=0.26]{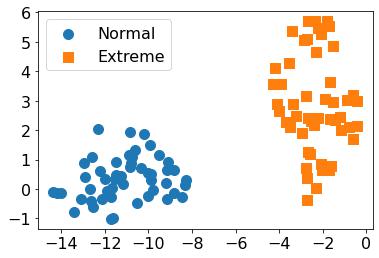}
        \caption{Ele-price}
        
    \end{subfigure}
    \hfill
    \begin{subfigure}[b]{0.38\columnwidth}
        \includegraphics[scale=0.26]{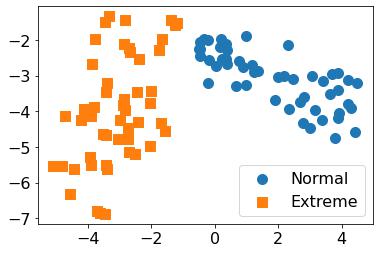}
        \caption{Pressure}
        
    \end{subfigure}
    \hfill
    \begin{subfigure}[b]{0.38\columnwidth}
        \includegraphics[scale=0.26]{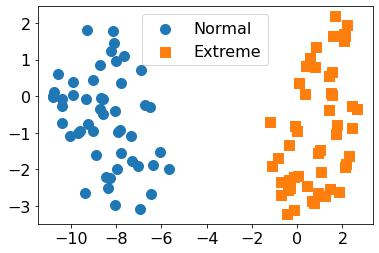}
        \caption{PM25}
        
    \end{subfigure}
    \hfill
    \begin{subfigure}[b]{0.38\columnwidth}
        \includegraphics[scale=0.26]{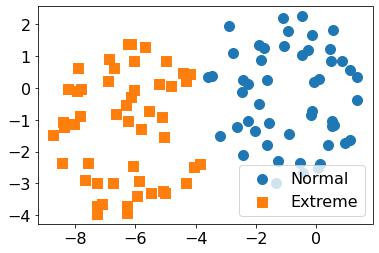}
        \caption{Ele-load}
        
    \end{subfigure}
    \hfill
    \begin{subfigure}[b]{0.38\columnwidth}
        \includegraphics[scale=0.26]{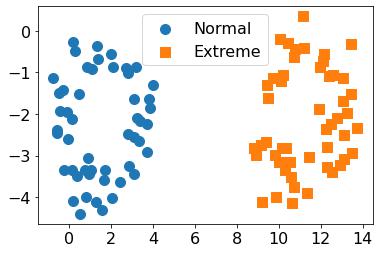}
        \caption{Water Level}
        
    \end{subfigure}
    
\caption{Embedding visualization. The blue circles and orange squares represent normal and extreme samples, respectively.}

\label{fig:embed_vis}
\end{figure*}
\begin{figure*}[ht!]
\centering
    \begin{subfigure}[b]{0.36\columnwidth}
        \includegraphics[scale=0.24]{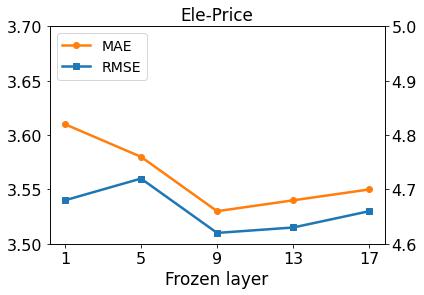}
    \end{subfigure}
    \hfill
    \begin{subfigure}[b]{0.36\columnwidth}
        \includegraphics[scale=0.24]{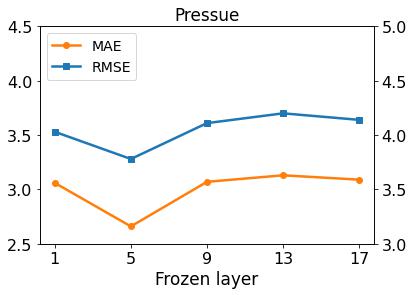}
    \end{subfigure}
    \hfill
    \begin{subfigure}[b]{0.35\columnwidth}
        \centering
        \includegraphics[scale=0.24]{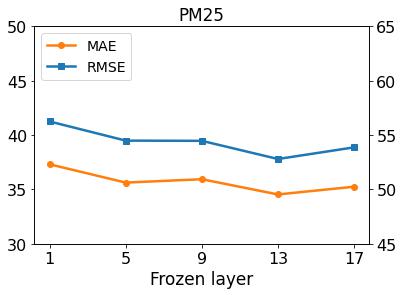}
    \end{subfigure}
    \hfill
    \begin{subfigure}[b]{0.38\columnwidth}
        \includegraphics[scale=0.24]{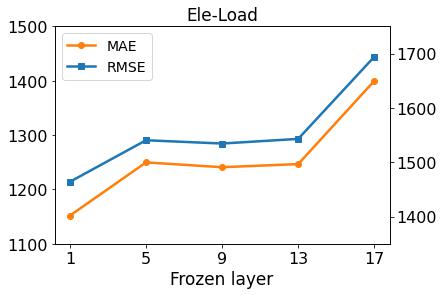}
    \end{subfigure}
    \hfill
    \begin{subfigure}[b]{0.37\columnwidth}
        \includegraphics[scale=0.24]{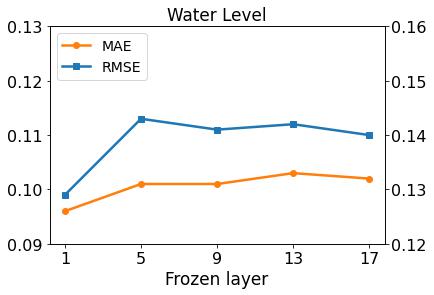}
    \end{subfigure}
    
\caption{Hyperparameter-tuning of the frozen layer. The left and right y-axis describe the MAEs and RMSEs, accordingly.}

\label{fig:hypers}
\end{figure*}

\subsection{Embedding Visualization}
To assess the effectiveness of reweighting in differentiating between extreme and normal samples, we employ t-distributed stochastic neighbor embedding (t-SNE) \cite{van2008visualizing} for a visual representation of the sample embedding. 
t-SNE is a dimensionality reduction technique that visualizes high-dimensional data in a lower-dimensional space \cite{liu2024timex}.
In Figure \ref{fig:embed_vis}, we randomly pick 50 extreme samples and 50 normal samples and visualize their embedding extracted from the last hidden layer. 
This visualization reveals a clear separation between normal and extreme samples in the embedding space, with samples of the same type tending to cluster together.

\subsection{Hyper-parameter Tuning}
As described in Sections \ref{sec:finetune} and \ref{sec:exp}, in the fine-tuning process, we freeze a subset of the lower layers and keep the remaining layers trainable with the $L2$ regularization.
We show the primary hyperparameter adjustments, (i.e., the number of frozen layers) in Figure \ref{fig:hypers}. 
We can observe that fine-tuning with different trainable parameters has varying effects. 
The selection of optimal hyperparameters requires a meticulous process of experimentation across datasets.

\subsection{Case study with model explanability}
We conduct a case study using the water level dataset to predict water levels by considering other covariates (e.g., precipitation). 
The normal and extreme training samples are separated by 95$^{th}$ percentile of the covariate (precipitation rate) in the data set. 
Figure \ref{fig:water_explain} shows that our reweighting and fine-tuning methods paid greater attention to extreme precipitation events. Note that this did not occur in the original unweighted model, which seemed to paint most attention values with a more uniform brush. 
\begin{figure}[ht]
\centering
    \includegraphics[width=0.8\columnwidth]{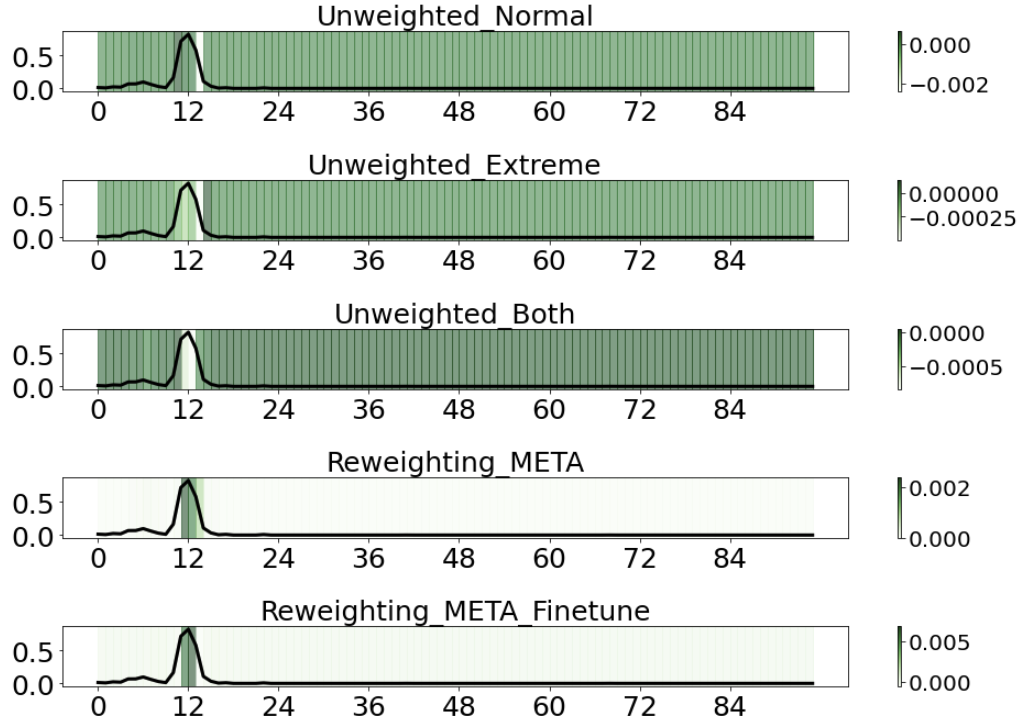} 
\caption{Precipitation explainability using LIME \cite{ribeiro2016should,shi2023power}.}
\label{fig:water_explain}
\end{figure}
\section{Discussion and Conclusions}
In this work, we tackle the challenge of predicting extreme events in time series. 
We introduce a reweighting technique as an initial solution, which is subsequently complemented by fine-tuning to further enhance performance. 
All three \emph{reweighting} methods prove effective. Meta-learning-based reweighting surpasses the other two heuristic methods, confirming the significant advantages of seeking optimal weights in an automated manner.
Our results show that models trained exclusively on normal or extreme samples are doomed by their distribution, demonstrating both normal and extreme samples are needed along with effective reweighting to establish foundational knowledge and get good performance.
\emph{Fine-tuning} can further boost the performance of two heuristic reweighting methods but is less effective sometimes. 
It is also worth noting that while the enhancement achieved through fine-tuning may appear modest, it holds significant value as it serves to \textbf{further} augment the already effective reweighting methods proposed in our work.
Last but not least, by using explainability techniques, we also demonstrate that the \emph{reweighting} and \emph{fine-tuning} approaches have achieved the task of paying prioritized attention to extreme events of input data, which is an important application in practice.

\section*{Acknowledgment}
This work is primarily supported by I-GUIDE, an Institute funded by the National Science Foundation, under award number 2118329.

\bibliographystyle{plain}
\bibliography{reference}

\end{document}